\documentclass[mlmain]{jmlr}

\usepackage{booktabs}
\usepackage{tikz}
\usetikzlibrary{arrows.meta,positioning}

\jmlrvolume{}
\firstpageno{1}
\jmlryear{2026}
\jmlrworkshop{Symmetry and Geometry in Neural Representations}
\jmlrpmlr
\jmlrauthors{Ridham Patel}

\newcommand{\R}{\mathbb{R}}
\newcommand{\SO}{\mathrm{SO}(3)}
\newcommand{\IKNO}{\textsc{IKNO}}

\title[Isotropic Kernels for Neural Dynamics]{Exact $\mathrm{SO}(3)$-Equivariant Isotropic Kernels for Rotation-Robust Neural Dynamics}

\author{
  \Name{Ridham Patel} \Email{ridham.patel@iitgn.ac.in}\\
  \addr Indian Institute of Technology Gandhinagar, Gandhinagar, India
}

\begin{document}

\maketitle

\begin{abstract}
Neural surrogates for vector-valued partial differential equations can fit
training data yet change their predictions when the same physical state is
expressed in a rotated coordinate frame. We study this failure on
three-dimensional Navier--Stokes dynamics observed at irregularly placed
points. We introduce the Invariant-Conditioned Isotropic Kernel Neural Operator
(\IKNO), a compact graph model that builds local interactions from scalar
quantities unchanged by rotation and vector directions that rotate with the
data. Consequently, rotating the positions and velocities rotates the
predicted velocity change in exactly the same way. On a held-out test set fixed
after model design, training unconstrained graph models on randomly rotated
examples reduces but does not eliminate their coordinate dependence. In
contrast, \IKNO{} is consistent to numerical precision, matches the forecasting accuracy of a
general rotation-aware Tensor Field Network with $5.6$ times fewer parameters,
and outperforms a parameter-matched graph simulator. These results show that a
compact, PDE-specialized model can remove coordinate dependence without
sacrificing forecasting accuracy.
\end{abstract}


\section{Introduction}
\label{sec:introduction}

The same physical state can be expressed in many coordinate frames. A quantity
is invariant if a transformation leaves it unchanged; a map is
equivariant if its output transforms with its input. For a
three-dimensional velocity field, rotation equivariance requires a joint
rotation of coordinates and velocities to rotate the predicted velocity change
by the same amount. The vector prediction changes, but the physical rule does
not depend on the axes. We study proper three-dimensional rotations, denoted
$\mathrm{SO}(3)$, which exclude reflections. Without this constraint, low error
in the training frame does not determine behavior in a new frame
\citep{villar2021scalars,helwig2023gfno}.

This issue is especially relevant for learned partial differential equation
(PDE) solvers on irregular meshes. Graph models are usually
permutation-equivariant: reordering sampled points merely reorders their
outputs. This lets graph simulators accommodate changing discretizations and
supports accurate long-horizon prediction
\citep{sanchez2020simulate,pfaff2021mesh}. However,
permutation equivariance concerns point labels, not spatial orientation, and
does not specify how a vector field transforms under rotation. Treating
Cartesian components as ordinary feature channels can therefore make the
learned dynamics depend on the axes
\citep{bronstein2021geometric}.

Rotation-equivariant architectures remove this dependence by prescribing how
features transform. Tensor-field and steerable networks propagate geometric
quantities according to their rotation rules
\citep{thomas2018tensor,fuchs2020se3}, while lighter-weight
constructions use vector-valued features or invariant scalar weights
\citep{satorras2021en,villar2021scalars}. We ask whether such an exact constraint improves forecasting
under a coordinate-frame shift without reducing accuracy in the observed
frame. We compare it with rotation augmentation, which trains a model on
randomly rotated copies of the observed states
\citep{brandstetter2022lie,puny2022frame}.

We use a controlled one-frame prediction experiment on three-dimensional
incompressible Navier--Stokes flow. The model maps a velocity field sampled at
irregular points to a one-step derivative estimate. Training and validation
use orientation-biased initial conditions; testing uses both the original
states and jointly rotated coordinates, velocities, and targets. Because the
rotated target is known exactly by changing coordinates, additional
disagreement isolates frame dependence without a second simulation.

To perform this test, we introduce the Invariant-Conditioned Isotropic Kernel
Neural Operator (\IKNO). Each edge coefficient is produced from scalar
rotation invariants. The resulting coefficients weight four vector bases that
rotate with the data. The construction has three properties relevant to the
experiment: it is exactly $\SO$-equivariant, it annihilates constant fields in
its difference channels, and it contains only $4{,}869$ trainable parameters.

The main result is that architectural equivariance and rotation augmentation
produce measurably different operators. Augmentation reduces the rotation
sensitivity of two graph controls but does not eliminate it; \IKNO{} remains
unchanged across 1,000 sampled coordinate frames. A Tensor Field Network
provides the complementary exact-equivariance control: it matches \IKNO{} in
forecasting accuracy but uses $5.6$ times more parameters. A parameter-matched
graph simulator is substantially less accurate. The evidence therefore
supports a specific conclusion: the specialized isotropic construction
provides exact rotation robustness at low parameter cost.

\section{Related Work}
\label{sec:related}

Equivariant representation learning formalizes weight sharing over transformations
\citep{cohen2016group,bronstein2021geometric}. Tensor Field Networks and
$\mathrm{SE}(3)$-Transformers use typed features that follow irreducible group
representations and spherical harmonics, a basis for angular dependence
\citep{thomas2018tensor,fuchs2020se3}; e3nn implements their tensor
products \citep{geiger2022e3nn}. Lighter Cartesian constructions include EGNN,
Vector Neurons, PaiNN, and SEGNN
\citep{satorras2021en,deng2021vector,schutt2021painn,brandstetter2022segnn}.
Classical isotropic representation theory and its modern scalar formulation
express equivariant vector functions through invariant coefficients and
equivariant generators \citep{smith1971isotropic,villar2021scalars}. The latter
also distinguishes proper rotations from reflections through cross-product
generators. \IKNO{} selects a compact, PDE-oriented subset of admissible
generators on an irregular mesh; it does not claim a complete parameterization.

Learned PDE models take several complementary architectural forms. Neural
operators learn maps between function spaces and can amortize repeated
PDE solution \citep{li2021fno}. Graph simulators instead represent local
physical interactions by message passing, which naturally supports particles
and irregular meshes \citep{sanchez2020simulate,pfaff2021mesh}. Message-passing
PDE solvers connect these learned updates to classical local discretizations
\citep{brandstetter2022message}. Symmetry has been incorporated through data
augmentation derived from Lie point symmetries \citep{brandstetter2022lie},
steerable differential operators \citep{jenner2022steerable}, and group
equivariant Fourier layers \citep{helwig2023gfno}. LieConv treats arbitrary
continuous data \citep{finzi2020lieconv}, while Neural ePDOs learn adaptive
equivariant differential operators \citep{he2023epdo}. In contrast to Fourier
methods on regular grids, \IKNO{} acts directly on irregular three-dimensional
samples. Our experiment also measures operator equivariance separately from
forecasting error, which distinguishes exact architectural consistency from
good performance on one orientation distribution.

\section{Problem Setting}
\label{sec:problem}

We first define the evolution problem and the symmetry to be tested. Let
$\{\boldsymbol{x}_i\}_{i=1}^N\subset\R^3$ be sample coordinates and let
$\boldsymbol{u}_i(t)\in\R^3$ denote velocity. A learned local operator
$f_\theta$ predicts the derivative
\begin{equation}
  \dot{\boldsymbol{u}}_i(t)
  = f_{\theta,i}(\boldsymbol{U}(t),\boldsymbol{X}),
  \qquad
  \boldsymbol{U}=(\boldsymbol{u}_1,\ldots,\boldsymbol{u}_N).
  \label{eq:dynamics}
\end{equation}
We integrate this derivative with an explicit Euler step of duration
$\Delta t$ to obtain autoregressive forecasts.

The relevant transformation rotates coordinates and vector components
together. For $R\in\SO$, define
$R\boldsymbol{X}=\{R\boldsymbol{x}_i\}_i$ and
$R\boldsymbol{U}=\{R\boldsymbol{u}_i\}_i$. The operator is
$\SO$-equivariant when
\begin{equation}
 f_\theta(R\boldsymbol{U},R\boldsymbol{X})
 =R f_\theta(\boldsymbol{U},\boldsymbol{X})
 \quad\text{for every }R\in\SO.
 \label{eq:equivariance}
\end{equation}
Equation~\ref{eq:equivariance} is stronger than low rotated prediction error.
The equation is an operator identity; rotated prediction error also depends on the difficulty
of the forecasting task. We therefore report both quantities.

Our physical system is the incompressible Navier--Stokes equation
\begin{align}
 \partial_t\boldsymbol{u}
 +(\boldsymbol{u}\cdot\nabla)\boldsymbol{u}
 &=-\nabla p+\nu\nabla^2\boldsymbol{u},
 & \nabla\cdot\boldsymbol{u}&=0,
 \label{eq:ns}
\end{align}
on a periodic three-dimensional domain. The differential terms in
Equation~\ref{eq:ns} are coordinate-covariant. Consequently, jointly rotating
coordinates and vector values gives a valid representation of the same local
dynamics, even when the observed initial-condition distribution favors one
orientation \citep{olver1993applications,brandstetter2022lie}.

\section{Invariant-Conditioned Isotropic Kernel}
\label{sec:method}

The method must use local directional information without tying a prediction
to a global axis. Motivated by isotropic representation results, we separate
each message into invariant coefficients and equivariant vector generators
\citep{smith1971isotropic,villar2021scalars}.

\begin{figure}[t]
\centering
\resizebox{\linewidth}{!}{%
\begin{tikzpicture}[
  >=Latex,
  box/.style={draw,rounded corners=1.5pt,align=center,inner sep=4pt,font=\small},
  inv/.style={box,fill=blue!6},
  vec/.style={box,fill=orange!8},
  flow/.style={-Latex,semithick},
  node distance=5mm and 7mm
]
\node[box] (edge) {local edge\\$\widehat{\boldsymbol r}_{ij},\boldsymbol u_i,
  \delta\boldsymbol u_{ij}$};
\node[inv,above right=4mm and 10mm of edge] (stats) {rotation invariants\\
  $\boldsymbol s_{ij}\in\mathbb R^6$};
\node[inv,right=of stats] (coeff) {scalar MLP $\phi_\theta$\\
  $a_{ij},b_{ij},c_{ij},d_{ij}$};
\node[vec,below right=4mm and 10mm of edge] (basis) {equivariant vector bases\\
  $\delta\boldsymbol u,\ \widehat{\boldsymbol r}(\widehat{\boldsymbol r}^{\!\top}\delta\boldsymbol u),$
  $\ \boldsymbol u_i,\ \boldsymbol u_i\!\times\!\widehat{\boldsymbol r}$};
\node[box,right=10mm of basis] (mix) {scalar--vector\\weighted sum $\boldsymbol m_{ij}$};
\node[box,right=9mm of mix] (out) {mean over $j$ and $+\lambda\boldsymbol u_i$\\
  $f_{\theta,i}(\boldsymbol U,\boldsymbol X)$};
\draw[flow] (edge) -- node[above left,font=\scriptsize] {norms, dot products} (stats);
\draw[flow] (edge) -- node[below left,font=\scriptsize] {typed vectors} (basis);
\draw[flow] (stats) -- (coeff);
\draw[flow] (coeff.east) -| (mix.north);
\draw[flow] (basis) -- (mix);
\draw[flow] (mix) -- (out);
\end{tikzpicture}%
}
\caption{\IKNO{} message construction. Invariant coefficients weight
equivariant vector bases; aggregation preserves their common rotation and
therefore yields an equivariant output.}
\label{fig:method}
\end{figure}
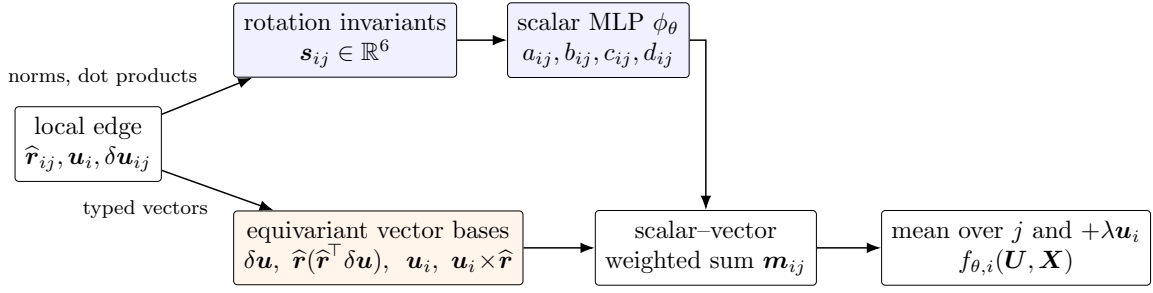

\subsection{Invariant edge description}

For a directed edge from node $i$ to node $j$, define the relative coordinate,
unit direction, and velocity difference
\begin{equation}
 \boldsymbol{r}_{ij}=\boldsymbol{x}_j-\boldsymbol{x}_i,
 \qquad
 \widehat{\boldsymbol{r}}_{ij}
 =\frac{\boldsymbol{r}_{ij}}{\|\boldsymbol{r}_{ij}\|_2},
 \qquad
 \delta\boldsymbol{u}_{ij}=\boldsymbol{u}_j-\boldsymbol{u}_i.
 \label{eq:edge}
\end{equation}
The edge graph connects points within radius $r_{\max}$ and retains at most
$K$ neighbors. Because this rule depends only on distance, a rotation does not
change the graph except for floating-point tie cases.

The coefficient network receives six scalar invariants:
\begin{equation}
 \boldsymbol{s}_{ij}=\big[
 \|\boldsymbol{r}_{ij}\|_2,
 \|\boldsymbol{u}_i\|_2,
 \|\delta\boldsymbol{u}_{ij}\|_2,
 \boldsymbol{u}_i^\top\widehat{\boldsymbol{r}}_{ij},
 \delta\boldsymbol{u}_{ij}^\top\widehat{\boldsymbol{r}}_{ij},
 \boldsymbol{u}_i^\top\delta\boldsymbol{u}_{ij}
 \big].
 \label{eq:invariants}
\end{equation}
Apart from radial scale $\|\boldsymbol r_{ij}\|_2$ and the fixed norm
$\|\widehat{\boldsymbol r}_{ij}\|_2=1$, these quantities encode the nonconstant
Gram information of
$(\widehat{\boldsymbol r}_{ij},\boldsymbol u_i,\delta\boldsymbol u_{ij})$.
Inner products generate rotation-invariant scalar descriptions of vector
inputs \citep{villar2021scalars}. A multilayer perceptron
$\phi_\theta:\R^6\rightarrow\R^4$ maps them to
$(a_{ij},b_{ij},c_{ij},d_{ij})$.

\subsection{Equivariant message and update}

Guided by the same representation results, we choose a compact PDE-oriented
subset of admissible $\SO$-equivariant generators rather than a universal
basis \citep{smith1971isotropic,villar2021scalars}:
\begin{align}
 \boldsymbol{m}_{ij}
 ={}&a_{ij}\delta\boldsymbol{u}_{ij}
 +b_{ij}\widehat{\boldsymbol{r}}_{ij}
       (\widehat{\boldsymbol{r}}_{ij}^{\top}\delta\boldsymbol{u}_{ij})
 +c_{ij}\boldsymbol{u}_i
 +d_{ij}(\boldsymbol{u}_i\times\widehat{\boldsymbol{r}}_{ij}).
 \label{eq:message}
\end{align}
The first two terms span direct and longitudinal difference responses. The
third permits state-dependent local dynamics. The cross product is
orientation-sensitive but rotation-consistent; as a pseudovector under
reflections, it distinguishes $\SO$ from $\mathrm O(3)$
\citep{villar2021scalars}.

Messages are mean-aggregated and combined with a scalar self-interaction:
\begin{equation}
 f_{\theta,i}(\boldsymbol{U},\boldsymbol{X})
 =\lambda\boldsymbol{u}_i
 +\frac{1}{|\mathcal{N}(i)|}
  \sum_{j\in\mathcal{N}(i)}\boldsymbol{m}_{ij}.
 \label{eq:update}
\end{equation}
The difference channels vanish on constant fields, matching the nullspace of
local differential responses \citep{brandstetter2022message,jenner2022steerable}.
The coefficient network has two hidden layers of 64 SiLU units. For optimization
stability, its final layer and $\lambda$ are initialized to zero, so training
begins from a stationary derivative.

\begin{proposition}
If $\mathcal{N}(i)$ depends only on pairwise distances, then for every
$R\in\SO$ and $\boldsymbol a\in\R^3$,
$f_\theta(R\boldsymbol U,R\boldsymbol X+\boldsymbol a)
=R f_\theta(\boldsymbol U,\boldsymbol X)$.
\end{proposition}

\begin{proof}
Under a joint rotation, each entry of $\boldsymbol{s}_{ij}$ is unchanged, so
the four scalar coefficients are unchanged. Each vector basis in
Equation~\ref{eq:message} is multiplied by $R$: this follows from preservation
of inner products and from
$(R\boldsymbol{a})\times(R\boldsymbol{b})=R(\boldsymbol{a}\times\boldsymbol{b})$
for $R\in\SO$. Distance-based neighborhoods are unchanged, and summation and
scalar self-interaction commute with $R$. Therefore each output is multiplied
by $R$. A global translation cancels in every
$\boldsymbol r_{ij}=\boldsymbol x_j-\boldsymbol x_i$, proving the translation
claim.
\end{proof}

\section{Experimental Design}
\label{sec:experiments}

The experiments test one question: does exact operator equivariance improve
prediction under an unseen coordinate frame while preserving useful accuracy
in the observed frame? We first describe the data and controls, then define
metrics that distinguish forecasting accuracy from operator consistency.

\subsection{Data and split protocol}

We generate three-dimensional incompressible flow on $[0,2\pi]^3$ using a
$16^3$ pseudo-spectral solver with two-thirds dealiasing, Leray projection, and
Strang-split diffusion \citep{canuto1988spectral,strang1968construction}. A
fixed-axis large-scale forcing is used for 50 warm-up
steps to create an orientation-biased initial state. Each trajectory adds an
independent $10\%$ divergence-free perturbation, chosen to diversify
trajectories while retaining the warm-up orientation bias, and is then evolved
without forcing or observation noise. This separation is deliberate: the forecast
dynamics obey rotational covariance, while the observed initial-condition
distribution supplies only a restricted set of orientations.

Each trajectory contains 20 snapshots separated by $\Delta t=0.02$. The
spectral velocity is trilinearly interpolated to a separately sampled irregular
cloud of 500 points in normalized coordinates. Training uses random 128-point
subsamples. The supervised derivative target is
$(\boldsymbol{U}(t+\Delta t)-\boldsymbol{U}(t))/\Delta t$. Evaluation uses a
fixed 128-point subsample per trajectory.

The development data contain 25 trajectories with seeds 342--366. The first
20 form the model-development pool: a deterministic split assigns 16 to
optimization and four to validation, and validation error alone selects
checkpoints. Five additional trajectories with seeds 362--366 screen which
new controls warrant final evaluation; no screening result appears in the
paper. The proposed architecture was frozen before generating the final test
set. The added controls were specified and screened only on development data,
then fixed before their single final evaluation. The reported test set contains
five trajectories with seeds 367--371 and new point clouds.

\subsection{Controls and optimization}

We compare with persistence and three learned architecture families.
Persistence copies
the initial field and establishes whether a learned derivative improves over
zero dynamics. Graph Network-based Simulator (GNS) uses six message-passing
layers following \citet{sanchez2020simulate}. The second control adapts the
EGNN implementation of \citet{satorras2021en} by passing the three velocity
components as node-feature channels; we call it component-EGNN. This is
an unconstrained vector-field control, not an equivariant implementation of
typed vector features: the original EGNN guarantee applies to coordinate
updates with invariant node features.

To distinguish the proposed specialization from general equivariant message
passing, we implement a Tensor Field Network (TFN)
\citep{thomas2018tensor} with e3nn \citep{geiger2022e3nn}. It carries even
scalars and odd vectors, and applies distance-conditioned tensor products with
spherical harmonics through degree one. The TFN is $\mathrm{O}(3)$-equivariant,
and therefore also $\SO$-equivariant. Unlike \IKNO{}, it propagates learned
scalar and vector hidden channels through tensor products.

We introduce two controls for possible confounds. First, GNS and
component-EGNN are retrained with a fresh Haar-random joint rotation of
coordinates, velocities, and derivative targets for every training sample. We
sample these rotations by Gaussian QR factorization
\citep{mezzadri2007random}. These variants test approximate symmetry learned
through augmentation \citep{brandstetter2022lie,puny2022frame}. Second,
a six-layer GNS with hidden width nine contains 5,352 parameters and controls
for the capacity difference between the original GNS and \IKNO{}.

All learned methods use the same radius $r_{\max}=0.4$, at most 64 neighbors,
the same trajectory split, and the same temporal-window sampling protocol.
\IKNO{} uses
4,869 parameters, TFN 27,297, component-EGNN 30,058, and GNS 63,174. We train
three seeds for 100 epochs with Adam and cosine learning-rate decay
\citep{kingma2015adam}, using
validation-selected checkpoints. The initial learning rate is
$3\times10^{-3}$ for \IKNO{} and $2\times10^{-3}$ for the controls. Additional
implementation details and the two-frame extrapolation control are reported in
Appendix~\ref{app:details}.

\subsection{Metrics and rotation intervention}

Forecasting error at horizon $h$ is the relative Frobenius error
\begin{equation}
 e_h=\frac{\|\widehat{\boldsymbol{U}}_h-\boldsymbol{U}_h\|_F}
 {\|\boldsymbol{U}_h\|_F}.
 \label{eq:predictionmetric}
\end{equation}
We report $h=1$ and $h=10$. To measure coordinate-frame generalization, we draw
1,000 fixed Haar-random rotations for one-step prediction and use the first 20
for ten-step rollouts, concentrating the denser stress test on the cheaper
one-step evaluation. The same rotations are applied to every method,
coordinate set, input velocity, and target velocity. The resulting rotated
error tests the same physical state in an unseen frame.

Average rotated error can conceal a sensitive coordinate frame. We therefore
also measure one-step degradation relative to the original frame. Here
$e_1(\boldsymbol X,\boldsymbol U)$ denotes the one-step error evaluated on the
specified coordinates and input velocity:
\begin{equation}
 \Delta_1(R)=e_1(R\boldsymbol{X},R\boldsymbol{U})
             -e_1(\boldsymbol{X},\boldsymbol{U}),
 \label{eq:degradation}
\end{equation}
and report its mean, 95th percentile, and sampled maximum over the 1,000
rotations. The maximum is a finite-sample stress test, not a certified
worst-case bound.

Forecasting error can remain small even when the learned derivative is not
equivariant. We therefore define our normalized operator equivariance diagnostic
\begin{equation}
 e_{\mathrm{eq}}=
 \mathbb{E}_{R}\left[
 \frac{\operatorname{RMS}\!\left(f_\theta(R\boldsymbol{U},R\boldsymbol{X})
 -Rf_\theta(\boldsymbol{U},\boldsymbol{X})\right)}
 {\operatorname{RMS}\!\left(f_\theta(R\boldsymbol{U},R\boldsymbol{X})\right)
 +\operatorname{RMS}\!\left(Rf_\theta(\boldsymbol{U},\boldsymbol{X})\right)}
 \right].
 \label{eq:eqmetric}
\end{equation}
A value near zero means that the learned derivative is coordinate-consistent:
rotating the input rotates the derivative by the same amount. We average each test metric over five trajectories
within a run and report mean $\pm$ sample standard deviation over three
training seeds.

\section{Results}
\label{sec:results}

\subsection{Main result: architecture and augmentation produce different symmetry}

The main experiment asks whether rotation augmentation recovers the same
operator-level consistency as an equivariant architecture. It does not.
Figure~\ref{fig:main}a shows normalized equivariance errors of $0.163$ for GNS
and $0.238$ for component-EGNN. Augmentation reduces these errors by $45\%$ and
$43\%$, respectively, but leaves residual errors of $0.090$ and $0.136$. In
contrast, \IKNO{} reaches $5.60\times10^{-8}$ and TFN reaches
$2.59\times10^{-7}$. These values match the numerical precision observed for
the theoretically equivariant implementations (Appendix~\ref{app:details}).

\begin{figure}[t]
  \centering
  \includegraphics[width=0.97\textwidth]{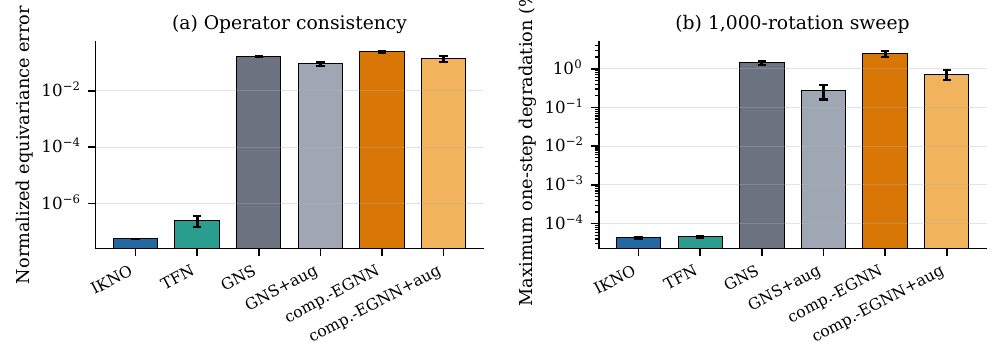}
  \caption{Architectural equivariance differs from rotation
  augmentation.
  Bars and error bars show mean and sample standard deviation over three
  training seeds on the post-design holdout. (a) Augmentation reduces but does
  not eliminate operator inconsistency; \IKNO{} and TFN are equivariant to
  numerical precision. (b) The sampled maximum of
  $\Delta_1(R)/e_1$ over 1,000 rotations remains nonzero after augmentation,
  whereas both exact architectures are stable up to numerical precision.}
  \label{fig:main}
\end{figure}

The 1,000-rotation sweep establishes that the residual operator error affects
prediction in particular frames. Figure~\ref{fig:main}b reports the sampled
maximum degradation as a percentage of original-frame error. The value is
$1.41\%$ for GNS and $2.45\%$ for component-EGNN. Augmentation reduces these
values to $0.27\%$ and $0.72\%$, but does not eliminate them. The corresponding
values are $4.3\times10^{-5}\%$ for \IKNO{} and $4.5\times10^{-5}\%$ for TFN.
Thus the distinction between approximate and exact operator consistency is
visible even when their average forecasting errors are close.

\begin{table}[t]
\floatconts
{tab:main}
{\caption{Post-design holdout means over three training seeds and five test
trajectories (lower is better); rotated $e_1/e_{10}$ use 1,000/20 fixed
rotations. The lower panel reports \IKNO{} relative to each baseline: positive
reductions and ratios above one favor \IKNO{}. Underlining marks its wins over
both augmented controls and, after rotation, every unconstrained control; the
single bold entry is the lowest $e_{\mathrm{eq}}$. Persistence has undefined
derivative-normalized equivariance error. Per-seed results are in
Appendix~\ref{app:perseed}.}}
{\scriptsize\setlength{\tabcolsep}{5.8pt}\begin{tabular}{lrrrrrr}
\toprule
Method & Params & $e_1$ & $e_{10}$ & Rot. $e_1$ & Rot. $e_{10}$ & $e_{\mathrm{eq}}$\\
\midrule
Persistence & 0 & .012219 & .120160 & .012219 & .120160 & ---\\
GNS & 63,174 & .011451 & .111186 & .011533 & .112025 & .163\\
GNS+aug. & 63,174 & .011516 & .111798 & .011507 & .111706 & .090\\
comp.-EGNN & 30,058 & .011418 & .110843 & .011584 & .112529 & .238\\
comp.-EGNN+aug. & 30,058 & .011505 & .111677 & .011526 & .111905 & .136\\
TFN & 27,297 & .011489 & .111487 & .011489 & .111487 & $2.59\!\times\!10^{-7}$\\
\IKNO{} (ours) & 4,869 & \underline{.011493} & \underline{.111488} & \underline{.011493} & \underline{.111488} & $\mathbf{5.60\!\times\!10^{-8}}$\\
\bottomrule
\end{tabular}

\vspace{5pt}
{\setlength{\tabcolsep}{3.1pt}\begin{tabular}{lrrrrrrr}
\toprule
Baseline & $\Delta e_1$ & $\Delta e_{10}$ & $\Delta$ Rot. $e_1$ & $\Delta$ Rot. $e_{10}$ & Param. ratio & Eq. ratio & Max-deg. ratio\\
\midrule
GNS & $-0.36\%$ & $-0.27\%$ & $+0.35\%$ & $+0.48\%$ & $12.97\!\times$ & $2.91\!\times\!10^6$ & $3.31\!\times\!10^4$\\
GNS+aug. & $+0.20\%$ & $+0.28\%$ & $+0.12\%$ & $+0.19\%$ & $12.97\!\times$ & $1.60\!\times\!10^6$ & $6.44\!\times\!10^3$\\
comp.-EGNN & $-0.65\%$ & $-0.58\%$ & $+0.79\%$ & $+0.92\%$ & $6.17\!\times$ & $4.25\!\times\!10^6$ & $5.73\!\times\!10^4$\\
comp.-EGNN+aug. & $+0.11\%$ & $+0.17\%$ & $+0.29\%$ & $+0.37\%$ & $6.17\!\times$ & $2.43\!\times\!10^6$ & $1.69\!\times\!10^4$\\
TFN & $-0.03\%$ & $-0.001\%$ & $-0.03\%$ & $-0.001\%$ & $5.61\!\times$ & $4.63\!\times$ & $1.06\!\times$\\
GNS-small & $+3.22\%$ & $+4.03\%$ & $+3.45\%$ & $+4.30\%$ & $1.10\!\times$ & $9.27\!\times\!10^6$ & $2.50\!\times\!10^4$\\
\bottomrule
\end{tabular}}
}
\end{table}

The lower panel makes the pairwise forecasting comparisons explicit rather
than retaining only global column winners. Component-EGNN is more accurate in
the original orientation, but \IKNO{} lowers its rotated errors at both
horizons. Against its augmented counterpart, \IKNO{} lowers all four
forecasting errors. These descriptive differences accompany reductions by
several orders of magnitude in operator inconsistency and sampled worst-frame
degradation.

Table~\ref{tab:main} answers a second question: does the specialized operator
sacrifice accuracy relative to a general equivariant network? TFN and \IKNO{}
differ by only $0.03\%$ at one step and less than $0.001\%$ at ten steps.
Exploratory two-sided paired $t$-tests give $p=0.706$ and $p=0.994$,
respectively. Their forecasting performance is therefore indistinguishable in
this experiment, while TFN uses $5.6$ times more parameters and spherical
harmonics. This comparison supports an efficiency claim, not an accuracy
advantage over TFN.

\subsection{Supporting result: controlling for parameter count}

The preceding comparison leaves open whether the large unconstrained controls
are disadvantaged because capacity permits orientation-specific fitting. We
isolate capacity with a six-layer, 5,352-parameter GNS. This parameter-matched
model obtains rotated errors of $0.011903$ at one step and $0.116500$ at ten
steps. \IKNO{} reduces these errors by $3.45\%$ and $4.30\%$, respectively;
exploratory paired tests across seeds give $p=0.014$ at both horizons. The corresponding
original-frame reductions are $3.22\%$ and $4.03\%$. The advantage therefore
persists when parameter count and message-passing depth are controlled.

\subsection{Supporting result: learning nontrivial dynamics}

Exact equivariance would be uninformative if obtained by predicting no
dynamics. Persistence provides this control because it is exactly equivariant
but copies the input field. \IKNO{} reduces persistence error by $5.95\%$ at
one step and $7.22\%$ at ten steps. Per-seed results in
Appendix~\ref{app:perseed} further show that the aggregate comparisons are not
driven by a single run.

\section{Discussion and Limitations}
\label{sec:discussion}

Low in-frame error measures fit to an observed orientation distribution,
whereas Equation~\ref{eq:eqmetric} tests whether the derivative is a
coordinate-consistent map. Augmentation improves both quantities but leaves
nonzero operator error and worst-frame degradation, consistent with the
distinction between sampled augmentation and exact constructions
\citep{brandstetter2022lie,puny2022frame}. \IKNO{} and TFN remove that
degradation by construction.

The scope of this conclusion is limited. First, the final test contains five
trajectories from one low-resolution Navier--Stokes generator. The evidence
does not establish superiority across PDEs, Reynolds numbers, geometries, or
boundary conditions. Second, the warm-up uses a fixed-axis forcing, whereas
forecast trajectories are unforced. This choice creates the orientation shift
needed for the controlled test, but it is not an isotropic sampling of initial
conditions. Third, arbitrary rotations change the coordinate description of
the periodic box; the test is a joint coordinate-frame transformation of
sampled states rather than a second solver run on an independently rotated
grid.

The baselines also delimit the claim. Component-EGNN is intentionally an
unconstrained adaptation for vector fields and should not be interpreted as a
failure of typed scalar--vector models such as PaiNN or SEGNN
\citep{schutt2021painn,brandstetter2022segnn}. The TFN supplies typed scalar and vector
features, but is restricted to degree-one spherical harmonics and one chosen
width; it does not exhaust the space of steerable architectures. Its matched
accuracy supports the parameter-efficiency of the specialized basis on this
task, not a general efficiency ordering. The two-frame extrapolator in
Appendix~\ref{app:details} is much more accurate because the trajectories are
temporally smooth; the present result is therefore specific to one-frame
operator identification, where the derivative must be inferred from spatial
structure.

\section{Conclusion}
\label{sec:conclusion}

\IKNO{} combines invariant scalar edge functions with an equivariant vector
basis to guarantee $\SO$ equivariance. On an untouched Navier--Stokes holdout,
it matches a tensor-field network with $5.6$ times fewer parameters, removes
coordinate dependence left by augmentation, and beats a parameter-matched
control. This supports compact, PDE-specialized exact operator consistency.

\bibliography{references}

\appendix

\section{Per-Seed Results}
\label{app:perseed}

Table~\ref{tab:perseed} reports the values underlying the aggregate main
result for the exact architectures, augmented controls, and parameter-matched
GNS. The same 1,000 one-step rotations and 20 ten-step rotations are used for
every method and seed.

\begin{table}[ht]
\floatconts
{tab:perseed}
{\caption{Per-seed post-design holdout results.}}
{\small\renewcommand{\arraystretch}{0.9}\begin{tabular}{llrrrrr}
\toprule
Method & Seed & $e_1$ & $e_{10}$ & Rot. $e_1$ & Rot. $e_{10}$ & $e_{\mathrm{eq}}$\\
\midrule
\IKNO{} & 0 & .011498 & .111562 & .011498 & .111562 & $5.65\times10^{-8}$\\
\IKNO{} & 1 & .011491 & .111461 & .011491 & .111461 & $5.52\times10^{-8}$\\
\IKNO{} & 2 & .011489 & .111442 & .011489 & .111442 & $5.64\times10^{-8}$\\
\midrule
TFN & 0 & .011480 & .111346 & .011480 & .111346 & $2.36\times10^{-7}$\\
TFN & 1 & .011499 & .111621 & .011499 & .111621 & $3.81\times10^{-7}$\\
TFN & 2 & .011489 & .111495 & .011489 & .111495 & $1.61\times10^{-7}$\\
\midrule
GNS + aug. & 0 & .011515 & .111767 & .011502 & .111657 & .074\\
GNS + aug. & 1 & .011532 & .111940 & .011510 & .111741 & .105\\
GNS + aug. & 2 & .011502 & .111686 & .011508 & .111719 & .091\\
\midrule
comp.-EGNN + aug. & 0 & .011501 & .111632 & .011541 & .112063 & .170\\
comp.-EGNN + aug. & 1 & .011516 & .111801 & .011520 & .111824 & .123\\
comp.-EGNN + aug. & 2 & .011498 & .111599 & .011518 & .111829 & .116\\
\midrule
GNS-small & 0 & .011904 & .116483 & .011902 & .116470 & .511\\
GNS-small & 1 & .011763 & .114880 & .011820 & .115504 & .476\\
GNS-small & 2 & .011958 & .117163 & .011988 & .117524 & .571\\
\bottomrule
\end{tabular}}
\end{table}

\FloatBarrier

\section{Implementation and Additional Controls}
\label{app:details}

All models follow the same optimization and selection protocol. Every training
sample contributes four randomly selected one-step temporal
windows per epoch. Models minimize mean-squared derivative error, clip gradient
norms at 10, and use weight decay $10^{-5}$. Checkpoint selection minimizes
validation one-step relative error. The test split is evaluated once after
freezing the architecture and selecting the checkpoint.

Two simple temporal controls characterize the predictability of the dataset.
Persistence receives one frame and obtains $e_1=0.012219$ and
$e_{10}=0.120160$. A two-frame linear extrapolator estimates the increment
from the first two observations and obtains $e_1=0.000259$ and
$e_{10}=0.014188$. The latter is not a like-for-like baseline because it
receives an additional frame that nearly reveals the time derivative. Its
performance shows that this dataset is temporally smooth and that the main
experiment specifically tests one-frame spatial operator learning.

Numerical equivariance can differ from the exact algebraic result because the
graph is constructed in floating-point arithmetic. The theoretical graph is
unchanged by rotation because pairwise distances are
preserved. In floating-point arithmetic, points exactly at the radius boundary
could change adjacency. No such instability is observed in the reported
graphs: \IKNO{} prediction error is unchanged to the displayed precision and
the residual operator error is on the order of single-precision rounding.

\end{document}